\documentclass{llncs}
\usepackage{amsfonts,amsmath}
\usepackage{xspace}                           
\usepackage{booktabs}                                     
\usepackage{graphics}
\usepackage{tikz}
\usepackage{dsfont}

\newcommand{\prob}[1]{\Pr\left(#1\right)}
\newcommand{\expect}[1]{\mathbf{E}\left[#1\right]}

\newcommand{\chilow}{\chi_{\text{low}}}
\newcommand{\chihigh}{\chi_{\text{high}}}

\newcommand{\psel}{\ensuremath{p_\mathrm{sel}}\xspace}
\newcommand{\pmut}{\ensuremath{p_\mathrm{mut}}\xspace}
\newcommand{\pmuta}{\ensuremath{p_\mathrm{mut}^\mathrm{adapt}}\xspace}
\newcommand{\pmutm}{\ensuremath{p_\mathrm{mut}^\mathrm{mix}}\xspace}
\newcommand{\mut}{\ensuremath{\mathrm{mut}}\xspace}

\newcommand{\chimax}{\ensuremath{\chi_{\max{}}}}

\newcommand{\hdist}{\ensuremath{\mathrm{H}}}

\DeclareMathOperator{\bin}{Bin}
\DeclareMathOperator{\unif}{Unif}

\newcommand{\strp}{\ensuremath{p}} 

\usepackage{mathrsfs}

\newcommand{\ab}{\hspace{0.125em}}                        
\newcommand{\ie}{\hbox{i.\ab e.}\xspace}                  
\newcommand{\eg}{\hbox{e.\ab g.}\xspace}                  

\usepackage[noend]{algorithm,algorithmic}
\floatname{algorithm}{Algorithm}

\newcommand{\leadingones}{\text{\sc LeadingOnes}\xspace}
\newcommand{\Lo}{\text{\sc Lo}\xspace}

\usepackage{enumitem}

\usepackage{color}

\usepackage{microtype}

\begin{document}

\title{Self-adaptation of Mutation Rates\\in Non-elitist Populations}
\titlerunning{Self-adaptation of Mutation Rates in Non-elitist Populations}

\author{
  Duc-Cuong Dang\inst{1},  Per Kristian Lehre\inst{1}}
\authorrunning{D.-C. Dang and P. K. Lehre}
\tocauthor{D.-C. Dang and P. K. Lehre}
\institute{
School of Computer Science, University of Nottingham, United Kingdom \\
  \email{\{duc-cuong.dang,PerKristian.Lehre\}@nottingham.ac.uk}}

\maketitle

\begin{abstract}
  The runtime of evolutionary algorithms (EAs) depends critically on
  their parameter settings, which are often
  problem-specific. Automated schemes for parameter tuning have been
  developed to alleviate the high costs of manual parameter tuning.
  Experimental results indicate that self-adaptation, where parameter
  settings are encoded in the genomes of individuals, can be effective
  in continuous optimisation.  However, results in discrete
  optimisation have been less conclusive. Furthermore, a rigorous
  runtime analysis that explains how self-adaptation can lead to
  asymptotic speedups has been missing.  This paper provides the first
  such analysis for discrete, population-based EAs. We apply
  level-based analysis to show how a self-adaptive EA is capable of
  fine-tuning its mutation rate, leading to exponential speedups over
  EAs using fixed mutation rates.
\end{abstract}

\section{Introduction}\label{sec:intro}

An obstacle when applying
Evolutionary Algorithms~(EAs) is that their efficiency depends
crucially, and sometimes unpredictably, on their parameter settings,
such as population size, and mutation rates.
\emph{Parameter tuning} \cite{bib:Eiben2007}, where the parameters are
fixed before running the algorithm, is the most common way of choosing
the parameters.
A weakness with parameter
tuning is that optimal parameter settings may depend
on the current state of the search
process.
In contrast, \emph{parameter control} allows the parameters to change during
the execution of the algorithm, \eg according to a fixed schedule
(e.g. as in simulated annealing),
through feedback from the search, or via self-adaptation
\cite{bib:Eiben2007}.
Adaptive parameters can be essential and advantageous (\eg
covariance-matrix adaptation \cite{bib:Hansen2001}) in continuous
search spaces.  In discrete spaces, it has been
shown that changing the mutation rate as a function
of the current fitness \cite{bib:Bottcher2010} can improve the runtime,
and the $1/5$-rule has been used to adapt the population size \cite{bib:Doerr2015}.

While previous studies have shown the benefit of adaptive parameters, only
global parameters were analysed.
Our focus is different, we look at
the so-called ``evolution of evolution'' or true self-adaptation
\cite{bib:Eiben2007}, in which the parameter is encoded in the genome
of individual solutions. As far as we know, the existing studies on
this topic from the EC literature is mostly experimental
\cite{bib:Back1992,bib:Eiben2007,bib:Rijn2015}, or about proving the
convergence of the population model at their limit \cite{bib:Back1992},
\ie infinite population.

We study how mutation rates can evolve within a \emph{non-elitist}
population, where the mutation rate of each individual is encoded by
its own genome. The rate at which the mutation rate mutates is
specified by a \emph{strategy parameter} $p$. In \emph{endogenous}
control, the strategy parameter is itself evolved
\cite{bib:Back1992,bib:Xue2015}. Here, we consider \emph{exogenous}
control of the strategy parameter $p$, where the value of the
parameter is fixed before the run.  Our contribution is twofold: using
\leadingones as a benchmark, we provide the necessary and sufficient
conditions, especially those on $p$, for self-adaptation to work; by
making a small modification of the function, we show that
self-adaptation is essential in optimising the modified function, more
precisely that a single mutation rate or uniform mixing of mutation
rates requires exponential time, while self-adaptation is
efficient. We also prove that a non-elitist EA can outperform the
($\mu$+$\lambda$) EA.

\section{Preliminaries}\label{sec:prelim}
For any $n\in\mathbb{N}$, define $[n]:=\{1,\dots, n\}$. The natural
logarithm is denoted by $\ln(\cdot)$, and the logarithm to the base
$2$ is denoted by $\log(\cdot)$. For $x \in \{0,1\}^n$, we
write $x(i)$ for the $i$-th bit value. The Hamming distance is
denoted by $\hdist(\cdot,\cdot)$ and the Iverson bracket by
$[\cdot]$.
Given a partition of a search space $\mathcal{X}$ into $m$ ordered
``levels'' $(A_1,\dots,A_m)$, we define $A_{\geq j}:=\cup_{i=j}^{m} A_i$.
A \emph{population} is a vector
$P\in\mathcal{X}^\lambda$, where the $i$-th element $P(i)$ is
called the $i$-th \emph{individual}. Given $A \subseteq \mathcal{X}$, we let
$|P \cap A|:= |\{i \mid P(i) \in A\}|$ be the number of
individuals in population $P$ that belong to the subset $A$.
All algorithms considered here are of the form of
Algorithm~\ref{algo:Algorithm1}
\cite{bib:Dang2014a}.  A new population $P_{t+1}$ is 
generated by independently sampling
$\lambda$ individuals from an existing population $P_t$ according to
$\psel$, and perturbing each of the sampled individuals by a variation
operator $\pmut$.  The selection mechanism $\psel$ implicitly embeds a
fitness function $g:\mathcal{Y}\rightarrow\mathbb{R}$.

\begin{algorithm}
  \caption{\cite{bib:Dang2014a}}
  \begin{algorithmic}[1]
    \REQUIRE Finite search space $\mathcal{Y}$ with an initial population $P_0\in \mathcal{Y}^\lambda$.
    \FOR{$t=0,1,2,\dots$ until a termination condition is met}
    \FOR{$i=1$ to $\lambda$}    
    \STATE Sample $I_t(i)\in[\lambda]$ according to $\psel(P_t)$, and set $x := P_t(I_t(i))$.\label{algo:Algorithm1:selection}
    \STATE Sample $x'\in\mathcal{Y}$ according to $\pmut(x),$ and set $P_{t+1}(i) := x'$. \label{algo:Algorithm1:psva-mutation}
    \ENDFOR
    \ENDFOR
  \end{algorithmic}
  \label{algo:Algorithm1}
\end{algorithm}

We consider the standard bitwise mutation operator, where for any
pair of bitstrings $x,x'\in\{0,1\}^n$ and any \emph{mutation rate}
$\chi \in (0,n]$, the probability of obtaining $x'$ from $x$ is
$
  \prob{x'=\mut(x, \chi)}  = \left(\chi/n\right)^{H(x,x')}\left(1-\chi/n\right)^{n-H(x,x')}.
$ 
To model the parameter control problem, we assume that 
Algorithm~\ref{algo:Algorithm1} must choose mutation rates
from a predefined set $\mathcal{M}$.

\textbf{Uniform mixing}, denoted $\pmutm$, chooses the mutation rate $\chi$
uniformly at random from the set $\mathcal{M}$ every time an
individual is mutated,
$
   \pmutm(x)
        := \mut(x, \chi), $
  where $\chi \sim\unif({\mathcal{M}})$.

The special case of $|\mathcal{M}| = 1$, i.e. a fixed mutation rate,
has been studied extensively
\cite{bib:Dang2014a,bib:Lehre2012a}. 

Here, we focus on $|\mathcal{M}|>1$. It is known
that such mixing of mutation operators can be beneficial
\cite{bib:Lehre2013,Doerr2015UnknownLength}.

\textbf{Self-adaptation} uses an extended search space
$\mathcal{Y} = \mathcal{X} \times \mathcal{M}$, where each element
$(x,\chi)$ consists both of a search point $x \in \mathcal{X}$ and 
a mutation rate $\chi \in \mathcal{M}$. A 
fitness function $g:\mathcal{Y}\rightarrow\mathbb{R}$ is defined by 
$g((x,\chi)) := f(x)$ for all 
$(x,\chi)\in \mathcal{Y}$. The mutation operator $\pmut$ is written
as $\pmuta$ and it is parameterised by a globally fixed parameter 
$\strp \in (0,1/2]$ such that
    $\pmuta((x,\chi)) :=  (x',\chi')$
   where $\chi'=\chi$ with probability $1-p$, and
         $\chi'\sim\unif(\mathcal{M}\setminus\{\chi\})$ otherwise, and
   $x' = \mut(x,\chi')$.

We analyse the runtime of Algorithm~\ref{algo:Algorithm1} using the 
level-based theorem \cite{bib:Corus2014}. This theorem applies to any population-based process
where the individuals in $P_{t+1}$ are sampled
independently from the same distribution $D(P_t)$, where $D$ 
maps populations $P_t$ to distributions over the search space $\mathcal{X}$.  In
Algorithm~\ref{algo:Algorithm1}, the map is
$D =\pmut \circ \psel$, i.e., composition of selection and mutation.

\begin{theorem}[\cite{bib:Corus2014}]\label{thm:level-based}
Given a partition
    $(A_1,\dots,A_{m+1})$ of $\mathcal{X}$, define $T :=
    \min\{t\lambda \mid |P_t\cap A_{m+1}|>0\}$ to be the first point
    in time that elements of $A_{m+1}$ appear in $P_t$ of Algorithm
    \ref{algo:Algorithm1}. If there exist parameters $z_1,\dots,z_m,
    z_*\in(0,1]$, $\delta>0$, a constant $\gamma_0 \in (0,1)$ and a
    function $z_0:(0,\gamma_0)\rightarrow\mathbb{R}$ such that for all
    $j\in[m]$, $P\in \mathcal{X}^\lambda$, $y\sim D(P)$ and $\gamma
    \in (0,\gamma_0]$ we have
  \begin{description}[noitemsep]
  \item[(G1)] $\displaystyle
    \Pr\left( y\in A_{\geq j}
      \mid |P\cap A_{\geq j-1}|\geq\gamma_0\lambda   
    \right)\geq z_j\geq z_*$
  \item[(G2)] $\displaystyle
    \Pr\left( y\in A_{\geq j}
      \mid |P\cap A_{\geq j-1}|\geq\gamma_0\lambda,
      |P\cap A_{\geq j}|\geq \gamma \lambda   
    \right)\geq z_0(\gamma)\geq (1+\delta)\gamma$
  \item[(G3)] $\displaystyle \lambda \geq
    \frac{2}{a}\ln\left(\frac{16m}{ac\varepsilon z_*}\right)$
    with $\displaystyle a = \frac{\delta^2 \gamma_0}{2(1+\delta)}$,
    $\varepsilon = \min\{\delta/2,1/2\}$
    and $c = \varepsilon^4/24$
  \end{description}
  then
  $\expect{T} \leq (2/c\varepsilon)(m\lambda(1+\ln(1+c\lambda)) +\sum_{j=1}^{m}1/z_j)$.
\end{theorem}

We apply the \emph{negative drift theorem for populations}
\cite{bib:Lehre2010} to obtain tail bounds on the runtime of
Algorithm~\ref{algo:Algorithm1}.  For any individual $P_t(i)$ in
Algorithm~\ref{algo:Algorithm1:psva-mutation} where $t\in\mathbb{N}$
and $i\in[\lambda]$, define
$R_t(i) := |\{j\in[\lambda]\mid I_t(j)=i\}|$, i.e., the number of
times the individual was selected. We define the \emph{reproductive
  rate} of the individual $P_t(i)$ to be $\expect{R_t(i)\mid P_t}$,
i.e., the expected number of offspring from individual
$P_t(i)$. Informally, the theorem states that if all individuals close
to a given search point $x^*\in\mathcal{X}$ have reproductive rate below
a certain threshold $\alpha_0$, then the algorithm needs exponential time to
reach $x^*$.  The threshold depends on the mutation
rate. Here, we derive a variant of this theorem for algorithms that
use multiple mutation rates. In particular, we assume that the
algorithm uses $m$ mutation rates, where mutation rate $\chi_i/n$ for
$i\in[m]$ is chosen with probability $q_i$. The proof of this theorem
is similar to that of Theorem~4 in \cite{bib:Lehre2010}, and thus
omitted

\begin{theorem}\label{thm:negdrift-generalised}
  For any $x^*\in\{0,1\}^n$, define $T :=\min \{t\mid x^*\in P_t \},$
  where $P_t$ is the population of Algorithm~\ref{algo:Algorithm1} at
  time $t\in\mathbb{N}$. If there exist constants
  $\alpha_0,c,c',\delta>0$ such that with
  probability $1-e^{-\Omega(n)}$
  \begin{itemize}
  \item the initial population satisfies $H(P_0,x^*)\geq c'n$ 
  \item for all $t\leq e^{cn}$ and $i\in[\lambda]$,
        if $H(P_t(i),x^*)\leq c'n,$ then the reproductive rate of 
        individual $P_t(i)$ is no more than $\alpha_0$,
  \item $\sum_{j=1}^m q_j e^{-\chi_j}\leq (1-\delta)/\alpha_0$, and
        $\max_j\chi_j\leq \chimax$ for a constant $\chimax$, 
  \end{itemize}
  then $\prob{T\leq e^{c''n}} = e^{-\Omega(n)}$ for a constant $c''>0$.
\end{theorem}
\begin{proof}[of Theorem~\ref{thm:negdrift-generalised}]
  We apply Theorem~1 in \cite{bib:Lehre2010}. The \emph{first condition}
  holds immediately. We use the distance function $g(x):=H(x,x^*)$
  over the interval $[b(n)/2,b(n)]$, where $b(n):=n/(e^{2\kappa}-1)$
  and
  $\kappa:=\max\left\{\ln(2),\ln\left(\frac{4\chimax}{\ln(1+\delta)}\right)\right\}$.
  Without loss of generality, we assume that $x^*=1^n$, hence $g(x)$
  is the number of 0-bits in $x$.

For the \emph{second condition}, the drift of the process 
$\Delta(i) := (g(X_{t+1})-g(X_t)\mid g(X_t)=i)$ is a sum
$\Delta(i)=\Delta^+(i)-\Delta^-(i)$ of two binomially distributed random variables,
the number of 1-bits flipped $\Delta^+(i)\sim\bin(n-i,Q)$,
and the number of 0-bits flipped, where $Q\sim q$. For $i<b(n)$,  we use $\exp(np(e^{\kappa }-1))$ as an upper
bound on the mgf of a binomially distributed random variable with 
parameters $n$ and $p$, and get
\begin{align*}
  \expect{e^{-\kappa \Delta(i)}\mid X_t} 
  & = 
    \expect{\expect{e^{-\kappa (\Delta^-(i)-\Delta^+(i))}\mid Q}\mid X_t} \\
  & \leq \sum_{j=1}^m q_j 
    \exp\left(  (n-i)(\chi_j/n)(e^{-\kappa}-1)
         + i(\chi_j/n)(e^\kappa-1)
        \right)\\
  & = \sum_{j=1}^m q_j 
    \exp\left(  -\chi_j(1-e^{-\kappa})
         + (i/n)\chi_j(e^\kappa- e^{-\kappa}))
        \right)\\
  & =  \sum_{j=1}^m q_j
    \exp\left( -\chi_j(1-2e^{-\kappa}) \right).
\end{align*}

Noting that $e^{-\kappa}\leq \ln(1+\delta)/(4\chimax)$, we get 
\begin{align*}
    \expect{e^{-\kappa \Delta(i)}\mid X_t} 
  & \leq \sum_{j=1}^m q_j \exp\left( -\chi_j\right)\exp\left(\chimax 2e^{-\kappa}) \right)\\
  & \leq \sum_{j=1}^m q_j \exp\left( -\chi_j\right)(1+\delta)^{1/2} 
    \leq \frac{1}{\alpha_0(1+\delta)^{1/2}}.
\end{align*}

The second condition is then satisfied. The \emph{third} and
\emph{fourth conditions} can be satisfied for any mutation rate $\chi/n$
for appropriate positive constants $\delta_2,\delta_3\in(0,1)$ and
$D(n)$, as long as $\kappa(n)\geq \ln(2)$ (see the proof of Theorem~4
in \cite{bib:Lehre2010}).\qed
\end{proof}

\section{General negative results}\label{sec:general-neg-results}

Using Theorem~\ref{thm:negdrift-generalised}, we can now show general negative 
results for uniform mixing and self-adaptation of two mutation rates for any 
function with a unique global optimum $x^*$, assuming that
the initial population is positioned sufficiently far away from $x^*$.
The following theorem is a special case of Theorem~\ref{thm:negdrift-generalised}
for $|\mathcal{M}| = 1$.

\begin{theorem}\label{thm:ineff-high-mut}
  The runtime of Algorithm~\ref{algo:Algorithm1} with reproductive rate $\alpha_0$
  and mutation rate $\chihigh/n\geq (\ln(\alpha_0) + \delta)/n$ for some constant
  $\delta>0$ satisfies $\prob{T\leq e^{cn}}=e^{-\Omega(n)}$
  on any function with a unique global optimum $x^*$
  assuming that
  $\hdist(P_0, x^*)\geq c'n$ for
  two constants $c>0$ and $c' \in (0,1)$.
\end{theorem}

For binary tournament and $(\mu,\lambda)$-selection, $\alpha_0$ is bounded from 
above by $2$ and $\lambda/\mu$ respectively. Hence, any mutation rate above 
$\ln(2)$ for 2-tournament selection and $\ln(\lambda/\mu)$ for 
$(\mu,\lambda)$-selection by a constant renders the EA inefficient.

For $|\mathcal{M}| = 2$, we have the following general result,
again due to Theorem~\ref{thm:negdrift-generalised}.
\begin{theorem}\label{thm:ineff-two-mut}
Consider Algorithm~\ref{algo:Algorithm1} with reproductive rate
$\alpha_0$  and
mutation rates 
$\chilow/n$ and $\chihigh/n$.
If there exist constants $\delta_1,\delta_2,\varepsilon>0$ such that
\begin{itemize}
  \item $\chilow \geq \ln(\alpha_0) - \ln(1 + \delta_1)$ and $\chihigh \geq \ln(\alpha_0) - \ln(1 - \delta_2)$,
  \item the EA chooses mutation rate $\chihigh$ with probability at least $\frac{\delta_1(1+\varepsilon)}{\delta_1+\delta_2}$,
\end{itemize}
then $\prob{T\leq e^{cn}}=e^{-\Omega(n)}$ on
any function with a unique  optimum $x^*$ given that 
$\hdist(P_0, x^*)\geq c'n$ for some 
constants $c',c>0$
\end{theorem}
\begin{proof}[of Theorem~\ref{thm:ineff-two-mut}]
  We have 
  \begin{align*}
    q_1e^{-\chilow} + q_2e^{-\chihigh} 
    & =  (1-q_2)e^{-\chilow}+q_2e^{-\chihigh} \\
    & \leq ((1-q_2)(1+\delta_1)+q_2(1-\delta_2))/\alpha_0\\
    & =    (1+\delta_1-q_2(\delta_1+\delta_2))/\alpha_0
      \leq (1-\varepsilon\delta_1)/\alpha_0,
  \end{align*}
  which by Theorem~\ref{thm:negdrift-generalised} implies the result.\qed
\end{proof}

Uniform mixing selects the mutation rate $\chihigh/n$ with probability 
$1/2$. Thus, if $\delta_1/(\delta_1+\delta_2)$ is below $1/2$ by a constant 
then the EA is inefficient. For example, in binary tournament, the setting 
$\chilow \geq \ln(3/2) - \ln(100/99)$ and $\chihigh \geq \ln{3} + \ln(33/32)$ 
satisfies the conditions of the theorem for $\delta_1 = 103/297$, $\delta_2 = 
105/297$ and $\delta_1/(\delta_1+\delta_2) = 103/208 < 1/2$. 
Theorem~\ref{thm:self-adapt-succeed} later on will show the efficiency of
self-adaptation in this setting in contrast to the uniform mixing.
In self-adaptation, $\chihigh/n$ is selected with at least probability $p$, thus 
self-adaptation does not work if $p$ is above $\delta_1/(\delta_1 + \delta_2)$ by 
a constant.

\section{Robust self-adaptation}\label{sec:lo-mixed-mutation-rates}

The previous section showed how critically non-elitist EAs depend on
having appropriate mutation rates. A slightly too high mutation rate
$\chihigh$ can lead to an exponential increase in runtime. Uniform
mixing of mutation rates can fail if the 
of mutation rates $\mathcal{M}$ contains one such high mutation rate,
even though the set also contains an
appropriate mutation rate $\chilow.$ 

Self-adaptation has a similar problem if the strategy parameter $p$ is
chosen too high. However, we will prove for a simple, unimodal fitness
function that for a sufficiently small strategy parameter $p$,
self-adaptation becomes highly robust, and is capable of fine-tuning
the mutation rate.  For the rest of this section, we consider a set of
two mutation rates $\mathcal{M}=\{\chilow,\chihigh\}$ which for
arbitrary parameters $\ell\in[n]$ and $\varepsilon>0$ are defined by
$
  \left(1-\frac{\chihigh}{n}\right)^\ell < \frac{\mu}{\lambda} 
  \leq \left(1-\frac{\chihigh}{n}\right)^{\ell-1} 
$ and
$  \frac{\mu}{\lambda}(1+\varepsilon) \leq \left(1-\frac{\chilow}{n}\right)^n.$
By the previous section, if $\ell$ is chosen
sufficiently small, and hence $\chihigh$ sufficiently high, then uniform
mixing will fail on any problem with a unique optimum. In contrast, using a 
Chernoff and a union bound, the following lemma shows that individuals that 
have chosen
$\chihigh$ will quickly vanish from a 
self-adapting population, and the population will be dominated by
individuals choosing the appropriate mutation 
parameter $\chilow$.

\begin{lemma}\label{lemma:lo-few-bad}
  Let $Y_t:=|P_t\cap A_{-1}|$ where $P_t$ is the population of
  Algorithm~\ref{algo:Algorithm1} at time $t\in\mathbb{N}$ 
  with $(\mu,\lambda)$-selection on \leadingones and the set 
  $A_{-1}$ is as defined in Eq. (\ref{eq:1}). Then 
  $
    \prob{Y_t\geq \max((3/4)\mu, (1-p/3)^tY_0)}\leq 
    t\cdot e^{-\Omega(\lambda)}
  $
  for all $t\in\mathbb{N}$.
\end{lemma}
\begin{proof}[of Lemma~\ref{lemma:lo-few-bad}]
  For an upper bound, we assume that search points in $B$ have higher
  fitness than search points outside $B$. The probability of producing
  a $B$-individual with $(\mu,\lambda)$-selection is at most
  \begin{align*}
    \left(1-\frac{\chihigh}{n}\right)^{\ell}
    \left( \frac{Y_t}{\mu}(1-p) + \left(1-\frac{Y_t}{\mu}\right)p \right)
    & \leq \left(\frac{\mu}{\lambda}\right)\left( \frac{Y_t}{\mu}(1-2p) + p \right)\\
    & \leq \left(\frac{\mu}{\lambda}\right)\left(\max\left(\frac{3}{4},\frac{Y_t}{\mu}\right)(1-2p) + p \right)\\
    & \leq \max\left(\frac{3\mu}{4\lambda},\frac{Y_t}{\lambda}\right)(1-2p/3)
      =: p_s
  \end{align*}
  Hence, $Y_{t+1}$ is stochastically dominated by a random variable $Z\sim\bin(\lambda,p_s)$.
  It now follows by a Chernoff bound that 
  \begin{align*}
    \prob{Y_{t+1}\geq \max\left(3\mu/4,(1-p/3)Y_t\right)}
    & \leq \prob{Z\geq \expect{Z}\left(1+\frac{p}{3-2p}\right)}\\
    & \leq \exp\left(-\frac{p^2\mu}{12(3-2p)}\right).
  \end{align*}
  The proof now follows by induction on $t$ and a union bound.\qed
\end{proof}

\begin{theorem}
  Algorithm~\ref{algo:Algorithm1} with $(\mu,\lambda)$-selection where $\lambda\geq
  c\ln(n)$ for a sufficiently large constant $c>0$, and self-adaptation
  from the set $\mathcal{M}=\{\chilow,\chihigh\}$ 
  using a sufficiently
  small constant strategy parameter $p$ satisfying $(1+\varepsilon)(1-p)\geq 1+p\varepsilon$
  has expected runtime
  $O(n\lambda\log(\lambda)+n^2)$ on
  \leadingones.
\end{theorem}
\begin{proof}
  We partition the search space into the following $n+2$ levels
  \begin{align}
    A_{j}   & := 
   \begin{cases}  
     \left\{ (x,\chihigh) \mid \Lo(y)\geq\ell \right\} & \text{ if } j=-1\\
     \left\{ (x,\chilow), (x,\chihigh) \mid \Lo(x)=j \right\} & \text{ if } 0\leq j\leq\ell-1\\
     \left\{ (x,\chilow) \mid \Lo(x)=j \right\} & \text{ if } \ell \leq j\leq n.
   \end{cases}\label{eq:1}
  \end{align}
  The special level $A_{-1}$ contains search points with too high
  mutation rate. We first estimate the expected runtime assuming that
  there are never more than $(3/4)\mu$ individuals in level $A_{-1}$.
  In the end, we will account for the generations where this
  assumption does not hold.

  We now show that conditions (G1) and (G2) of the level-based theorem
  hold for the parameters $\gamma_0:=(1/8)(\mu/\lambda),$
  $\delta:=p\varepsilon,$ and $z_j=\Omega(1/n)$.  Assume that the
  current population has at least $\gamma_0\lambda=\mu/8$ individuals
  in $A_{\geq j-1}$ and $\gamma\lambda<\gamma_0\lambda$ individuals in
  $A_{\geq j}$, for $0\leq j\leq n$ and $\gamma\in[0,\gamma_0)$. If
  $0\leq j\leq \ell-1$, then an individual can be produced in levels
  $A_{\geq j}$ if one of the $\gamma\lambda$ individuals in these
  levels is selected, and none of the first $j$ bits are
  mutated. Assuming in the worst case that the selected individual has
  chosen the high mutation rate, the probability of this event is at
  least
  $
     (\frac{\gamma\lambda}{\mu})
    \left(
    \left(1-\frac{\chihigh}{n}\right)^{j}(1-p)+
    \left(1-\frac{\chilow}{n}\right)^{j}p
    \right)\\
     >
    (\frac{\gamma\lambda}{\mu})
    \left(
    \left(1-\frac{\chihigh}{n}\right)^{\ell-1}(1-p)+
    \left(1-\frac{\chilow}{n}\right)^{n}p
    \right)
    \geq \gamma (1+p\varepsilon).
  $ 
  All individuals in levels $j\geq\ell$ use the low mutation rate.
  Hence, an individual in levels $A_{\geq j}$ can be produced by selecting
  one the $\gamma\lambda$ individuals in this level, not change the
  mutation rate, and not flip any of the first $j\leq n$ leading
  1-bits. The probability of this event is at least
  $
    \frac{\gamma\lambda}{\mu}
    \left(1-\frac{\chilow}{n}\right)^{j}(1-p)
     >
      \frac{\gamma\lambda}{\mu}
    \left(
      \frac{\mu}{\lambda}(1+\varepsilon)(1-p)
    \right)
    \geq \gamma(1+\delta).
  $ 
  Condition (G2) is therefore satisfied for all levels. For condition
  (G1), assume that the population does not contain any individuals in
  $A_{\geq j}$. Then in the worst case, it suffices to select one of
  the at least $\gamma_0\lambda$ individuals in level $A_j$, switch
  the mutation rate, and only flip the first 0-bit and no other
  bits. The probability of this event is higher than
  $
    \frac{\gamma_0\lambda}{\mu}\left(\frac{\chilow}{n}\right)\left(1-\frac{\chihigh}{n}\right)^{n-1}p=\Omega(1/n).
  $

  Condition (G3) holds for any population size
  $\lambda\geq c\ln(n)$ and a sufficiently large constant $c$, because
  $\gamma_0$ and $\delta$ are constants. It follows that the expected
  number of generations until the optimum is found is
  $t_1(n)=O(n\log(\lambda)+n^2/\lambda)$. By Markov's inequality, the
  probability that the algorithm has not found the optimum after
  $2t_1(n)$ generations is less than $1/2$.

  Finally, we account for the generations with
  more than $(3/4)\mu$ individuals in level $A_{-1}$.  We
  call a phase \emph{good} if after $t_0(n)=O(\log(\lambda))$
  generations and for the next $2t_1(n)$ generations, there are fewer
  than $(3/4)\mu$ individuals in level $A_{-1}$. By
  Lemma~\ref{lemma:lo-few-bad}, a phase is good with probability
  $1-(t_0(n)+2t_1(n))\cdot e^{-\Omega(\lambda)}=\Omega(1)$, for
  $\lambda\geq c\ln(n)$ and $c$ a sufficiently large constant. By the
  level-based analysis, the optimum is found with probability at least
  $1/2$ during a good phase.  Hence, the expected number of phases
  required to find the optimum is $O(1).$ The theorem now follows by
  keeping in mind that each generation costs $\lambda$ evaluations.\qed
\end{proof}

We have shown that the EA can self-adapt to choose the low
mutation 
parameter 
$\chilow$ when required. Nevertheless, uniform mixing of mutation 
rates with a sufficiently small $\chilow$ could achieve the same 
asymptotic performance. Furthermore, naively picking a mutation rate
from the beginning also has a constant probability of optimising the 
function in polynomial time.
       Our aim is therefore to show that there exists a setting for which
       all the above approaches, except self-adaptation, fail.
To prove this, we have identified a problem $f_m$ where 
a high mutation rate is required in one part of the 
search space, and a low mutation rate is required in another part.
For $1\leq m<n$, define $f_m(0^n):=m$ and $f_m(x):=\leadingones(x)$
       for all $x\neq 0^n$.
We call the local optimum $0^n$ the \emph{peak}, and assume that all 
individuals in the initial population are peak individuals. 
It is clear that the elitist algorithm ($\mu$+$\lambda$)~EA without any
diversity mechanism will only accept a search point if it has at least
$m$ leading 1-bits. 
\begin{theorem}\label{thm:elitist-ea}
Starting at $0^n$, the ($\mu$+$\lambda$)~EA has expected
runtime $n^{\Omega(m)}$ on $f_m$.
\end{theorem}

To reach the optimal search point more efficiently, it is necessary to
accept worse individuals into the population, e.g. a non-elitist
selection scheme should be investigated. 
Since $f_m$ has a unique global optimum, either using only a too high mutation 
rate or uniformly mixing a correct mutation rate with a too high one can lead
to exponential runtime as discussed above.
Analogously to the $(\mu+\lambda)$~EA, we also prove that using a too low 
mutation rate fails because the population is trapped on the peak (\eg 
due to Theorem~\ref{thm:negdrift-generalised},
individuals fell off the peak have too low reproductive rate to optimise
$m$ leading $1$-bits).
Subsequent proofs use the two functions $q(i)
:= (1 - \chilow/n)^i$ and $r(i) := (1 - \chihigh/n)^i$, which are the probabilities of 
not flipping the first $i \in [n]$ bits using mutation rate $\chilow/n$ and 
$\chihigh/n$ respectively. Clearly, $q(i)$ and $r(i)$ are monotonically 
decreasing in $i$. We also use the function $\beta(\gamma):=2\gamma(1-\gamma/2)$,
which is the probability that binary tournament selection chooses one
of the $\gamma\lambda$ fittest individuals.

\begin{theorem}\label{thm:low-mut-fail}
  The runtime of Algorithm \ref{algo:Algorithm1} on $f_m$ with tournament size $2$, 
  initialised with the population at $0^n$ and with fixed mutation rate 
  $\chi\leq \ln(3/2)-\varepsilon$ for any constant $\varepsilon\in(0,\ln(3/2))$ 
  satisfies $\prob{T\leq e^{cn}}=e^{-\Omega(\lambda)}$ for a 
  constant $c>0$.
\end{theorem}

\begin{proof}[of Theorem~\ref{thm:low-mut-fail}]
  We will prove that with probability $1-e^{-\Omega(\lambda)}$, all
  individuals during the first $e^{cn}$ generations have less than $m$
  leading 1-bits, where $c>0$ is a constant. Clearly, this 
  stronger statement implies the theorem.

  Choose the parameter $\delta\in(0,1)$ such that
  $\ln((1+\delta)/(1-\delta))= \varepsilon$.  We first
  show by
  induction that with probability $1-e^{-\Omega(\lambda)}$, there are
  at least $(\lambda/2)(1+\delta/2)$ peak-individuals in each of the
  first $e^{c\lambda}$ generations, and we call the run of the
  algorithm a \emph{failure} otherwise. By Lemma~\ref{lemma:exp-bound}, the probability 
  of not mutating any bits when $n\in\mathbb{N}$ is sufficiently large is
  $
    q(n)
    \geq e^{-\chi}(1-\delta) 
    \geq (2/3)(1+\delta).
  $ 
  Assume that there are $\gamma\lambda\geq \lambda/2$ peak
  individuals in the current population. A peak individual is produced
  if a peak individual is selected and none of its bits are flipped.
  The probability of this event is at least 
  $
    \beta(\gamma) q(n) 
    \geq 
    \beta(1/2)q(n)
    = (1/2)(2-1/2)(2/3)(1+\delta)=(1/2)(1+\delta),
  $ 
  where the first inequality uses the fact that $\beta(\gamma)$ is
  strictly increasing in $\gamma$ over the interval $[0,1]$.  Hence,
  by a Chernoff bound, the probability that the next generation
  contains less than $(\lambda/2)(1+\delta/2)$ peak individuals is
  $e^{-\Omega(\lambda)}$. By induction and a union bound, the bound
  $e^{-\Omega(\lambda)}$ also holds for the next $e^{c\lambda}$
  generations, if $c>0$ is a sufficiently small constant.

  We now assume that the run is not a failure. Furthermore, we assume
  that the algorithm is optimising the function $g(x):=\min(m,f_m(x))$
  instead of $f_m$.
  Clearly, the time to reach at least $m$ leading 1-bits is the same,
  whether the algorithm optimises $g$ or $f_m$.  
  Assuming that there are
  more than $(\lambda/2)(1+\delta/2)$ peak individuals, the
  reproductive rate of any non-peak individual is always less than
  $
    \lambda\left(
    2
      (1/\lambda)\left(1-1/\lambda
    - (1/2)(1+\delta/2)
    \right)
    +\left(1/\lambda\right)^2
    \right)
    <  1-\delta/2 =: \alpha_0.
  $

  For non-peak individuals, the last $n-m$ bit-positions are
  irrelevant when the algorithm optimises $g$. We can therefore apply
  the negative drift theorem (Theorem~\ref{thm:negdrift-generalised}) 
  with respect to the
  algorithm limited to the first $m$ bit positions only. The variation
  operator in this algorithm flips each of the $m$ bits independently
  with probability $\chi'/m$, where $\chi'=\chilow(m/n)$. Hence, we 
  have $e^{-\chi'} < 1 = (1-\delta/2)/\alpha_0$, and the conditions of 
  the theorem are satisfied.\qed
\end{proof}

\begin{theorem}\label{thm:self-adapt-succeed}
  If $\mathcal{M}=\{\chilow,\chihigh\}$ where
  $\chilow:=\ln(3/2)-\varepsilon$ 
  for any constant $\varepsilon \in (0,\ln(100/99))$,
  and $\ln(3)\leq \chihigh = O(1)$,
  then there exists an $m\in\Theta(n)$ such that Algorithm~\ref{algo:Algorithm1} 
  starting with the population at $0^n$,
  with tournament size $2$,
  population size $\lambda \geq c\ln n$ for some constant $c>0$
  and self-adaptation of $\mathcal{M}$
  with $\strp=1/20$
  has expected runtime
  $O(n\lambda\log(\lambda)+n^2)$ on $f_m$.
\end{theorem}

Our intuition is that with sufficiently high mutation rate, some individuals 
fall off the peak and form a sub-population which optimises the $\leadingones$ 
part of the problem. This will happen if the selective pressure is not too high. 
However, at the same time, the population should be able to reach the optimal 
search point $1^n$ after escaping the local optimum. Here we used 
the level-based technique to infer constraints on the mutation rates and the
strategy parameter $\strp$. The proof idea follows closely from these observations.

We will need the following result to limit the number of individuals at 
unfavourable portions of the search space, \ie too many individuals in 
those portions will prevent the algorithm from moving in the right direction.
\begin{lemma}\label{lemma:escape-prob}
  Given any subset $A\subset\mathcal{X}$, let $Y_t := |P_t\cap A|$ be
  the number of individuals in generation $t\in\mathbb{N}$ of
  Algorithm~\ref{algo:Algorithm1} with tournament size $2$,
  that belong
  to subset $A$.  If there exist three parameters $\rho,\sigma,\varepsilon\in(0,1)$ such that
  $\prob{\pmut(y)\in A}\leq\rho$ for all $y\in A$ and
  $\prob{\pmut(y)\in A}\leq \sigma\gamma_*-\varepsilon$ for all
  $y\not\in A$,
  where $\gamma_*:=2-(1-\sigma)/\rho$,
  then $\prob{Y_t\geq \max\left(\gamma_*\lambda,(1-\varepsilon/2)^t
    Y_0\right)}\leq t\cdot e^{-\Omega(\lambda)}$.
\end{lemma}
\begin{proof}[of Lemma~\ref{lemma:escape-prob}]
  For an upper bound, we assume that all search points in $A$ have
  higher fitness than search points in $\mathcal{X}\setminus A$. The
  probability of selecting an individual in $A$ is therefore $\beta(Y_t/\lambda)$.
  The probability that any given offspring in 
  generation $t+1\leq e^{c\lambda}-1$ belongs to subset
  $A$ is no more than
  \begin{align*}
    \beta(Y_t/\lambda) \rho + \sigma\gamma_*-\varepsilon
    & \leq    \beta(\max\left(\gamma_*,Y_t/\lambda\right))\rho  + \sigma\gamma_*-\varepsilon\\
    & \leq    2\max\left(\gamma_*,Y_t/\lambda\right)(1-\max\left(\gamma_*,Y_t/\lambda\right)/2)\rho  + \sigma\gamma_*-\varepsilon\\
    & \leq    2\max\left(\gamma_*,Y_t/\lambda\right)(1-\gamma_*/2)\rho  + \sigma\gamma_*-\varepsilon\\
    & =       \max\left(\gamma_*,Y_t/\lambda\right)(1-\sigma) + \sigma\gamma_*-\varepsilon\\
    & \leq    \max\left(\gamma_*,Y_t/\lambda\right)(1-\varepsilon) =: p_s.
  \end{align*}
  Hence, $Y_{t+1}$ is stochastically dominated by the random variable $Z\sim\bin(\lambda,p_s)$.
  It now follows by a Chernoff bound that
  \begin{align*}
    \prob{Y_{t+1}\geq \max(\gamma_*\lambda,Y_t(1-\varepsilon/2))}
    & \leq \prob{Z\geq \max(\gamma_*\lambda,Y_t(1-\varepsilon/2))}\\
    & \leq   \prob{Z\geq \expect{Z}\left(1+\frac{\varepsilon }{2(1-\varepsilon)}\right)} \\
    & \leq \exp\left(-\frac{\varepsilon^2\max\left(\gamma_*\lambda,Y_t\right)}{12 (1-\varepsilon)}\right)\\
    &  \leq \exp\left(-\frac{\varepsilon^2\gamma_*\lambda}{12 (1-\varepsilon)}\right).
  \end{align*}
  The proof is completed by induction with respect to $t$ and a union bound.\qed
\end{proof}


\begin{proof}[of Theorem~\ref{thm:self-adapt-succeed}]
We apply the level-based theorem with respect to a 
partitioning of the search space $\mathcal{X}=\{0,1\}^n\times\mathcal{M}$
into the following $n+2$ levels
\begin{align*}
  A_j    & := 
\begin{cases}
\{(0^n,\chilow), (0^n,\chihigh) \}          & \text{ if } j=-1,\\
\{(x,\chilow),(x,\chihigh)\mid \Lo(x)=0\wedge x\neq 0^n \} & \text{ if } j=0,\\
\{(x,\chilow),(x,\chihigh)\mid \Lo(x)=j \} & \text{ if } 1\leq j\leq\ell-2,\\
\{(x,\chilow), (y, \chihigh)\mid \Lo(x)=\ell-1, \Lo(y)\geq \ell-1 \}     & \text{ if } j=\ell-1,\\ 
\{(x,\chilow)\mid \Lo(x)=j \}              & \text{ if } \ell\leq j\leq n.
\end{cases}
\end{align*}
where $\ell\in[n]$ is the unique integer such that 
$
    \left(1-\frac{\chihigh}{n}\right)^{\ell}
  < \frac{85}{171} \leq 
    \left(1-\frac{\chihigh}{n}\right)^{\ell-1}.
$ 
Note that as long as $m \leq \ln(171/85)(n-1)/\chihigh$, we have
$
  \left(1-\frac{\chihigh}{n}\right)^{m}
    \geq \left(e^{-\chihigh}\right)^\frac{m}{n-1}
    \geq \frac{85}{171} 
    > \left(1-\frac{\chihigh}{n}\right)^{\ell},
$ 
hence $\ell > m$. 

We first estimate the expected runtime assuming that 
every population contains less than
$\psi\lambda$ individuals in $A_{-1}$, and less than $\xi\lambda$
individuals in the set $B :=\{(y,\chihigh)\mid\Lo(y)\geq\ell\}$, where 
$\psi:=123/250$ and $\xi:=1/5$. In the end, we will account for the 
generations where these assumptions do not hold.
We begin by showing that condition (G2) of the level-based theorem 
hold for all levels.

\underline{Levels $0\leq j\leq m$}: 
Assume that the population contains $\gamma\lambda$ individuals in
levels $A_{\geq j}$ for any $\gamma\in(0,\gamma_0)$. An individual in
$A_{\geq j}$ will be selected if the tournament contains at least one
individual in $A_{\geq j}$, and no individuals in level $A_{-1}$.
The probability of this event is
$
  \beta(\gamma) 
                 \geq 2\gamma(1-\gamma_0/2-\psi).
$ 
The mutated offspring of the selected individual will belong to levels
$A_{\geq j}$ if none of the first $j\leq m$ bits are flipped, which
occurs with probability at least $r(m)$. Hence, condition (G2) 
is satisfied if there exists a $\gamma_0\in(0,1)$ and a constant 
$\delta>0$ such that for all $\gamma\in(0,\gamma_0]$, it holds
$
  \beta(\gamma)r(m)\geq \gamma(1+\delta),
$ 
i.e., it is sufficient to choose 
$m\in\mathbb{N}$ sufficiently small such that
$
  r(m)=\left(1-\frac{\chihigh}{n}\right)^m\geq \frac{1+\delta}{2(1-\gamma_0/2-\psi)}.
$ 
Note that such an $m=\Theta(n)$ exists, because
$2(1-\gamma_0/2-\psi)=127/125-\gamma_0>
1+\delta$ when $\gamma_0$ and $\delta$ are sufficiently small.

\underline{Levels $m+1\leq j<\ell$}: 
The probability of mutating an individual from $A_{\geq j}$ 
into $A_{\geq j}$, pessimistically assuming that the selected 
individual uses the high mutation rate $\chihigh$, is at least
$
     r(\ell-1)(1-p)+q(\ell-1)p
   > r(\ell-1)(1-p)+q(n)p   
   > (85/171)(1-p) + (2/3)p 
    = 1/2+1/180.
$ 
Hence, assuming that the current population has $\gamma\lambda$ 
individuals in $A_{\geq j}$ where $\gamma\in(0,\gamma_0)$, the
probability of selecting one of these individuals and mutating them
into $A_{\geq j}$ is at least
$
  \beta(\gamma)(r(\ell-1)(1-p) +q(\ell-1)p) 
  > 2\gamma(1-\gamma_0/2)(1/2+1/180) 
  = \gamma(1-\gamma_0/2)(1 + 1/90)
     > \gamma(1+\delta')
$ 
for some $\delta'>0$ given that $\gamma_0$ is a
sufficiently small constant. Note that the lower bound on 
$\beta(\gamma)$ here does not depend on $\psi$, and nor on $\xi$ 
because in this setting the peak individuals have lower fitness than 
the individuals in $A_j$, and $B \subset A_{\geq j}$.

\underline{Levels $\ell\leq j\leq n$:} By the level-partitioning, any
individual in these levels uses the low mutation rate $\chilow,$ and
other individuals with at least $\ell$ leading 1-bits belong to the
set $B$.  Assume that the current population contains
$\gamma\in(0,\gamma_0)$ individuals in levels $A_{\geq j}$. An
individual in $A_{\geq j}$ can be produced by having a binary
tournament with at least one individual from $A_{\geq j}$ and none of
the at most $\xi\lambda$ individuals in $B$, not mutating any
of the bits, and not changing the mutation rate. The probability of
this event is at least
$
  2\gamma(1-\gamma_0/2-\xi)q(n)(1-p) 
   \geq
   \gamma(4/5-\gamma_0/2)(19/15)
  = \gamma(1+1/75-(19/30)\gamma_0)
  > \gamma(1+\delta')
$ 
for some constant $\delta'>0$, assuming that $\gamma_0$ is sufficiently
small.

We now show that condition (G1) of the level-based theorem is satisfied
for a parameter $z=\Omega(1/n)$ in any level $j$. Assume that the
current population contains at least $\gamma_0\lambda$ individuals in
$A_{\geq j}$. Then, to create an individual in  $A_{\geq j+1}$, it is
sufficient to create a tournament of two individuals from
$A_{\geq j}$, flip at most one bit, and either keep or switch the
mutation rate. The probability of such an event is at least
$
  \gamma_0^2(\chilow/n)(1-\chihigh/n)^{n-1}p=\Omega(1/n).
$

To complete the application of the level-based theorem, we note that
since $\delta$ and $\gamma_0$ are constants, condition (G3) is
satisfied when $\lambda \geq c\ln n$ for some constant $c$. Hence, under 
the assumptions on the number of individuals in level $A_{-1}$ and $B$ 
described above, the level-based theorem implies that the algorithm obtains 
the optimum in expected $t_1(n)=O(n\log(\lambda)+n^2/\lambda)$
generations. Furthermore, by Markov's inequality, the probability that
the optimum has not been found within $2t_1(n)$ generations is less
than $1/2$.

To complete the proof, we justify the assumption that less than
$\psi\lambda$ individuals belong to level $A_{-1}$, and less than
$\xi\lambda$ individuals belong to $B$. We will show using
Lemma~\ref{lemma:escape-prob} that starting with any population, these
assumptions hold after an initial phase of $t_0(n)=O(\log(\lambda))$
generations. We call a phase \emph{good} if the assumptions
hold for the next $t_1(n)<e^{c\lambda}$ generations.

To apply Lemma~\ref{lemma:escape-prob} with respect to level $A_{-1}$,
we note that the probability of obtaining an individual in $A_{-1}$ by
mutating an individual in $A_{-1}$ is bounded from above by
$
  q(n)(1-p)+r(n)p \leq (2/3)e^{\varepsilon}(1-p)+p/3 \leq 65/99.
$ 
Furthermore, to mutate an individual from  $\mathcal{X}\setminus A_{-1}$ into $A_{-1}$, it is
necessary to flip at least one specific bit-position, i.e., with
probability $O(1/n)$. Therefore, by Lemma~\ref{lemma:escape-prob} with
$\sigma=49/4950$ and $\rho=65/99$, it holds for all $t$ where
$t_0(n)<t<e^{cn}$ and $t_0(n)=O(\log(\lambda))$ that
$
  \prob{|P_t\cap A_{-1}|\geq \psi\lambda } = e^{-\Omega(\lambda)}
$ 
where $\psi:=123/250$. 

Similarly, the probability of not destroying a $B$-individual
with mutation is by definition of $\ell$ at most
$
  \left(1-\frac{\chihigh}{n}\right)^{\ell}(1-p)
  \leq \left(\frac{85}{171}\right)\left(\frac{19}{20}\right)
  = \frac{17}{36} =: \rho.
$ 
To create a $B$-individual from $\mathcal{X}\setminus B$,
it is in the best case necessary to change the mutation rate from
$\chilow$ to $\chihigh$ and not mutate the first $\ell$
bit-positions. The probability of this event is
$
  \left(1-\frac{\chihigh}{n}\right)^{\ell}p 
  \leq \left(\frac{85}{171}\right)\left(\frac{1}{20}\right)
  = \frac{17}{684}.
$ 
Therefore, by Lemma~\ref{lemma:escape-prob} with respect to 
$\sigma:=3/20$ and the above value of $\rho$, 
for every generation $t$ where $t_0(n)<t<e^{c\lambda}$ and $t_0(n)=O(\log(\lambda))$
it holds
$
  \prob{|P_t\cap B|\geq \xi\lambda } = e^{-\Omega(\lambda)},
$
where $\xi:=1/5$.

To summarise, starting from any configuration of the population, a
phase of length $t_0(n)+2t_1(n)=O(n\log(\lambda)+n^2/\lambda)$
generations is \emph{good} with probability
$1-e^{-\Omega(\lambda)}$. If a phase is good, then the optimum will be
found by the end of that phase with probability at least $1/2$.  Hence,
the expected number of phases required to find the optimum is $O(1)$,
and the theorem follows, keeping in mind that each generation costs
$\lambda$ function evaluations.\qed
\end{proof}
Below are results from 1000 experiments
with the self-adaptive EA on the \leadingones function for $n=200$,
$p=1/1000$ using $(\mu,\lambda)$-selection
for $\mu=500$, $\lambda=4\mu$, and mutation 
parameters 
$\mathcal{M}=\{2/5,2\}$.
%
For each $j\in[n]$, the figure contains a box-plot describing the
distribution of the fraction of the population choosing $\chilow$ over 
all generations where the $(1/10)$-ranked individual in the population has 
$j$ leading one-bits.
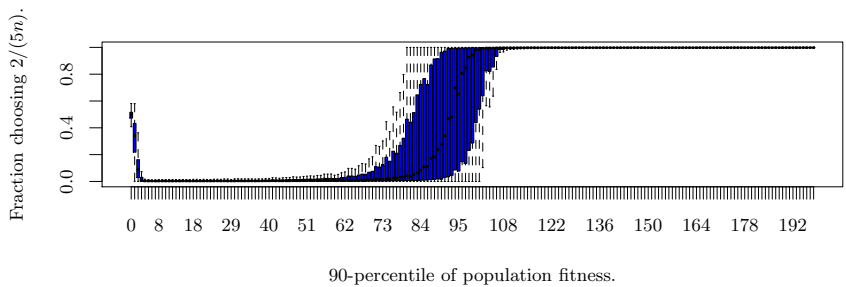
\begin{figure}[h]
  \scalebox{0.78}{\input{02feb2016.tex}}  
\end{figure}

    The initial population, including mutation rates, are sampled
    uniformly at random. Hence the $(1/10)$-ranked individual will have
    fitness close to $1$ in the first generations. For $j\leq 5$, \ie
    early in the run, approximately half of the population chooses the low
    mutation.
    However, the population quickly switches to the higher
    mutation 
    $\chihigh$ until the $(1/10)$-ranked individual in the
    population reaches a value approximately $j\geq 60$ where the
    population switches to the lower mutation 
    $\chilow$. Almost all individuals
    choose 
    $\chilow$ for $j\geq 108$.
These experimental results confirm that the population adapts the
mutation rate according to the region of the
fitness landscape currently searched.

\section{Conclusion}\label{sec:concl}

    This is the first rigorous runtime analysis of
    self-adaptation. We have demonstrated that 
    self-adaptation with a sufficiently low strategy parameter can robustly control the mutation-rates
    of non-elitist EAs in discrete search spaces, and
    that this automated control can lead to exponential speedups 
    compared to EAs that use fixed mutation rates, or 
    uniform mixing of mutation rate. 

{\small {\bf Acknowledgements} This work received funding from the
    European Union Seventh Framework Programme (FP7/2007-2013) under grant
    agreement no.\ 618091 (SAGE).}


\newpage
\appendix
\section*{Appendix A}

\begin{lemma}\label{lemma:exp-bound}
  For any $\delta\in(0,1)$ and $\chi>0$, if $n\geq
  (\chi+\delta)(\chi/\delta)$ then
  \begin{align*}
   (1-\delta)e^{-\chi} \leq \left(1-\frac{\chi}{n}\right)^n \leq e^{-\chi}.
  \end{align*}
\end{lemma}
\begin{proof}
  The upper bound follows immediately from the inequality $1+x\leq e^x$.
  For the lower bound, note first that $\ln(1-\delta)<-\delta$, hence
  \begin{align*}
    \left(\frac{n}{\chi}-1\right)(\chi-\ln(1-\delta))
    \geq n+\frac{n\delta}{\chi}-(\chi+\delta) \geq n.
  \end{align*}
  By making use of the fact that $(1-1/x)^{x-1}\geq 1/e$ and
  simplifying the exponent $n$ as above
  \begin{gather*}
    \left(1-\frac{\chi}{n}\right)^n 
    \geq \left[\left(1-\frac{\chi}{n}\right)^{(n/\chi)-1}\right]^{\chi-\ln(1-\delta)}
    \geq (1-\delta)e^{-\chi}. \;\qed
  \end{gather*}
\end{proof}

\begin{lemma}\label{lem:bin-process}
Let $X \sim \mathrm{Bin}(\lambda,p)$ with $p \leq (k/\lambda)(1 - \delta)$ for some $k \in [\lambda]$ and some $\delta \in [0,1)$, then $$\prob{X > k} \leq \exp\left(-\frac{k^2\delta^2}{2\lambda}\right).$$
\end{lemma}
\begin{proof}
Define $Y := \lambda - X$, thus $Y \sim \mathrm{Bin}(\lambda, q := 1 - p)$. Therefore, $q \geq \lambda - (k/\lambda)(1 - \delta) = (\lambda - k(1-\delta))/\lambda$ and
\begin{align*}
  \expect{Y} 
    \geq \lambda - k(1 - \delta).
\end{align*}

Then by a Chernoff bound,
\begin{align*}
\prob{X > k } 
  &=    \prob{Y < \lambda - k} \\ 
  &=    \prob{Y < \left(\lambda - k(1 - \delta)\right) 
                     \left(\frac{\lambda - k}{\lambda - k(1 - \delta)}\right)} \\
  &\leq \prob{Y < \expect{Y}
                     \left(1 - \frac{k\delta}{\lambda - k(1 - \delta)}\right)} \\
  &\leq \exp\left(- \left(\frac{k\delta}{\lambda - k(1 - \delta)}\right)^2 
                         \frac{\expect{Y}}{2}\right) \\
  &\leq \exp\left(- \left(\frac{k\delta}{\lambda - k(1 - \delta)}\right)^2 
                          \frac{\lambda - k(1 - \delta)}{2}\right) \\
  &=    \exp\left(- \frac{k^2\delta^2}{2(\lambda - k(1 - \delta))} \right) 
   <    \exp\left(- \frac{k^2\delta^2}{2\lambda} \right). \;\qed   
\end{align*}
\end{proof}


\end{document}